\documentclass[final]{IEEEtran}
\usepackage{bbm}
\usepackage[mathscr]{eucal}
\usepackage[cmex10]{amsmath}
\usepackage{epsfig,epsf,psfrag}
\usepackage{amssymb,amsmath,amsthm,amsfonts,latexsym,bm}
\usepackage{amsmath,graphicx,bm,xcolor,url}
\usepackage[caption=false]{subfig} 
\usepackage{fixltx2e}
\usepackage{array}
\usepackage{verbatim}
\usepackage{bm}
\usepackage{algorithmic, cite}
\usepackage{algorithm}
\usepackage{verbatim}
\usepackage{textcomp}
\usepackage{mathrsfs}
\usepackage{multirow}
\usepackage{epstopdf}

\catcode`~=11 \def\UrlSpecials{\do\~{\kern -.15em\lower .7ex\hbox{~}\kern .04em}} \catcode`~=13 

\allowdisplaybreaks[1]
 
\newcommand{\nn}{\nonumber}

\newcommand{\calA}{\mathcal{A}}
\newcommand{\calB}{\mathcal{B}}

\newcommand{\calE}{\mathcal{E}}
\newcommand{\calF}{\mathcal{F}}

\newcommand{\calN}{\mathcal{N}}

\newcommand{\calS}{\mathcal{S}}

\newcommand{\calU}{\mathcal{U}}
\newcommand{\calV}{\mathcal{V}}

\newcommand{\bbE}{\mathbb{E}}

\newcommand{\bbN}{\mathbb{N}}

\newcommand{\bbP}{\mathbb{P}}

\newcommand{\bbR}{\mathbb{R}}

\DeclareMathAlphabet{\mathbsf}{OT1}{cmss}{bx}{n}
\DeclareMathAlphabet{\mathssf}{OT1}{cmss}{m}{sl}

\DeclareSymbolFont{bsfletters}{OT1}{cmss}{bx}{n}  
\DeclareSymbolFont{ssfletters}{OT1}{cmss}{m}{n}
\DeclareMathSymbol{\bsfGamma}{0}{bsfletters}{'000}
\DeclareMathSymbol{\ssfGamma}{0}{ssfletters}{'000}
\DeclareMathSymbol{\bsfDelta}{0}{bsfletters}{'001}
\DeclareMathSymbol{\ssfDelta}{0}{ssfletters}{'001}
\DeclareMathSymbol{\bsfTheta}{0}{bsfletters}{'002}
\DeclareMathSymbol{\ssfTheta}{0}{ssfletters}{'002}
\DeclareMathSymbol{\bsfLambda}{0}{bsfletters}{'003}
\DeclareMathSymbol{\ssfLambda}{0}{ssfletters}{'003}
\DeclareMathSymbol{\bsfXi}{0}{bsfletters}{'004}
\DeclareMathSymbol{\ssfXi}{0}{ssfletters}{'004}
\DeclareMathSymbol{\bsfPi}{0}{bsfletters}{'005}
\DeclareMathSymbol{\ssfPi}{0}{ssfletters}{'005}
\DeclareMathSymbol{\bsfSigma}{0}{bsfletters}{'006}
\DeclareMathSymbol{\ssfSigma}{0}{ssfletters}{'006}
\DeclareMathSymbol{\bsfUpsilon}{0}{bsfletters}{'007}
\DeclareMathSymbol{\ssfUpsilon}{0}{ssfletters}{'007}
\DeclareMathSymbol{\bsfPhi}{0}{bsfletters}{'010}
\DeclareMathSymbol{\ssfPhi}{0}{ssfletters}{'010}
\DeclareMathSymbol{\bsfPsi}{0}{bsfletters}{'011}
\DeclareMathSymbol{\ssfPsi}{0}{ssfletters}{'011}
\DeclareMathSymbol{\bsfOmega}{0}{bsfletters}{'012}
\DeclareMathSymbol{\ssfOmega}{0}{ssfletters}{'012}

\newcommand{\eps}{\varepsilon}

\newcommand{\bone}{\mathbf{1}}

\newtheorem{theorem}{Theorem} 
\newtheorem{lemma}[theorem]{Lemma}

\newtheorem{definition}[theorem]{Definition}

\newtheorem{remark}[theorem]{Remark}

\newcommand{\qednew}{\nobreak \ifvmode \relax \else
      \ifdim\lastskip<1.5em \hskip-\lastskip
      \hskip1.5em plus0em minus0.5em \fi \nobreak
      \vrule height0.75em width0.5em depth0.25em\fi}

\usepackage{xspace}
\usepackage[ colorlinks = true,
             linkcolor = blue,
             urlcolor  = blue,
             citecolor = red,
             anchorcolor = green,
]{hyperref}

\usepackage{cite}
\allowdisplaybreaks[1]
\usepackage{amssymb}
\usepackage{pifont}

\begin{document}
\title{Optimal Best-Arm Identification under Fixed Confidence with Multiple Optima}
\author{Lan V.\ Truong,  {\em Senior Member, IEEE}
         
\thanks{The author is with the Faculty of Computer Science and Engineering, Ho Chi Minh City University of Technology (HCMUT), Vietnam and with Vietnam National University Ho Chi Minh City (VNU-HCM), Vietnam. Email: \url{lantv@hcmut.edu.vn.}
}}

\maketitle

\begin{abstract}
We study best-arm identification in stochastic multi-armed bandits under the fixed-confidence setting, focusing on instances with multiple optimal arms. Unlike prior work that addresses the unknown-number-of-optimal-arms case, we consider the setting where the number of optimal arms is known in advance. We derive a new information-theoretic lower bound on the expected sample complexity that leverages this structural knowledge and is strictly tighter than previous bounds. Building on the Track-and-Stop algorithm, we propose a modified, tie-aware stopping rule and prove that it achieves asymptotic instance-optimality, matching the new lower bound. Our results provide the first formal guarantee of optimality for Track-and-Stop in multi-optimal settings with known cardinality, offering both theoretical insights and practical guidance for efficiently identifying any optimal arm.
\end{abstract}
\begin{IEEEkeywords}
Multi-Armed Bandits, Best-Arm Identification, Sample Complexity.
\end{IEEEkeywords}
\section{Introduction}

The multi-armed bandit (MAB) problem provides a fundamental framework for sequential decision-making under uncertainty, where a learner interacts with a set of arms corresponding to unknown reward distributions. The objective is to efficiently balance exploration and exploitation in order to make statistically reliable decisions.

In the best-arm identification (BAI) problem under the fixed-confidence setting, the goal is to identify an arm with the highest (or nearly highest) expected reward with a prescribed confidence level, using as few samples as possible. This formulation arises in various applications such as clinical trials, A/B testing, hyperparameter optimization, and recommendation systems.

Most existing works on BAI assume the existence of a unique best arm. However, in many real-world scenarios, multiple optimal arms may exist, each achieving the same maximal expected reward. Degenne and Koolen \cite{Degenne2019PureEW} studied this case and derived a fundamental lower bound together with the sticky Track-and-Stop algorithm, which attains asymptotic sample complexity matching this bound when the number of optimal arms is \emph{unknown}. Despite these advances, the optimal strategy and the fundamental performance limits when the number of optimal arms is \emph{known} a priori remain unexplored. 

\subsection{Motivation}

The best-arm identification problem is central to sequential decision-making, where the objective is to identify an optimal arm using as few samples as possible while ensuring a fixed level of confidence. This problem arises in a wide range of domains, including clinical trials, A/B testing, and hyperparameter tuning. While the case of a unique best arm has been extensively studied, many practical applications naturally involve multiple arms achieving the same maximal expected reward. This motivates the need to understand the optimal sample complexity in this more general multi-optimal scenario.

Among existing algorithms for fixed-confidence BAI, the \emph{Track-and-Stop} algorithm stands out for its principled foundation based on likelihood ratio statistics and its strong empirical performance. Although it is widely regarded as near-optimal in the single-optimum case, a complete theoretical justification of its optimality when multiple optimal arms exist has remained elusive.  

Degenne and Koolen \cite{Degenne2019PureEW} addressed the setting where the number of optimal arms is unknown, establishing a lower bound and proposing the sticky Track-and-Stop algorithm, which asymptotically achieves this bound. In contrast, this paper focuses on the complementary setting where the number of optimal arms is known in advance. Under this assumption, we derive a new information-theoretic lower bound that characterizes the minimal achievable sample complexity, which is strictly tighter than the bound in \cite{Degenne2019PureEW}.  

Furthermore, we propose a modified version of the Track-and-Stop algorithm and prove that it achieves \emph{instance-optimal sample complexity} in the fixed-confidence regime. These results provide the first formal characterization of the optimal sample complexity when multiple optimal arms exist and their number is known, thereby completing the theoretical picture and offering new insights into the design of exploration strategies in stochastic bandit problems.

\subsection{Related Work}
Best-arm identification (BAI) in stochastic multi-armed bandits has been a central focus in the literature for over a decade, particularly under fixed-confidence and fixed-budget settings. In the fixed-confidence framework, the goal is to identify an arm with the highest expected reward with probability at least $1-\delta$, while minimizing the expected number of samples \cite{even2006action, garivier2016optimal}. The fixed-budget setting, by contrast, fixes the total sampling budget and seeks to minimize the probability of misidentification \cite{gabillon2012best}.

Early work introduced elimination-based strategies and stopping rules that guarantee PAC (Probably Approximately Correct) bounds \cite{even2006action}, later complemented by information-theoretic approaches establishing matching lower and upper bounds under the assumption of a unique best arm \cite{garivier2016optimal}. Such analyses leverage Kullback–Leibler (KL) divergences to characterize instance difficulty, and algorithms like Track-and-Stop \cite{garivier2016optimal} and lil’UCB \cite{jamieson2014lil} exploit these insights for near-optimal performance. Other PAC-based algorithms, including LUCB \cite{kalyanakrishnan2012pac}, adapt confidence bounds to efficiently balance exploration, but rely on uniqueness of the best arm and may over-sample tied arms.

In practice, multiple arms often share the maximal expected reward, e.g., in clinical trials, recommendation systems, or A/B testing. Standard algorithms may waste samples distinguishing statistically equivalent arms, motivating approaches that exploit ties. Prior work has addressed related problems, such as top-k arm identification and best-subset selection \cite{gabillon2012best, chen2014combinatorial, jamieson2014bandit}, which aim to identify all optimal arms but often incur unnecessary sample complexity when the goal is to find any optimal arm. Degenne and Koolen \cite{Degenne2019PureEW} explicitly considered multiple optimal arms, deriving a lower bound and proposing the sticky Track-and-Stop algorithm, which achieves asymptotically optimal sample complexity when the number of optimal arms is unknown.

From a theoretical standpoint, classical lower bounds assume a unique best arm \cite{garivier2016optimal}; handling multiple optima requires refined bounds that characterize minimal sampling effort across tied arms. Beyond fixed-confidence BAI, related challenges arise in fixed-budget BAI \cite{gabillon2012best, chen2014combinatorial}, combinatorial bandits \cite{chen2014combinatorial}, and contextual bandits \cite{simchi2021best}, where tie-aware strategies improve both efficiency and fairness \cite{xu2021fairness}.

In summary, while most BAI studies focus on unique optima, recent work highlights the importance of efficiently handling multiple optimal arms. Our paper contributes by deriving a new lower bound for the case where the number of optimal arms is known and demonstrating that a modified Track-and-Stop algorithm achieves instance-optimal sample complexity, closing a key theoretical gap between the known- and unknown-cardinality settings.
\subsection{Contributions}

This paper makes the following contributions:
\begin{itemize}
	\item \textbf{Tighter Fundamental Limit:} We derive a new information-theoretic lower bound on the sample complexity of fixed-confidence BAI when the number of optimal arms is known, strictly improving over previous bounds for the unknown-cardinality setting \cite{Degenne2019PureEW}.
       \item \textbf{Tie-Aware Track-and-Stop Algorithm:} We propose a modification of the Track-and-Stop algorithm with a tie-aware stopping rule that leverages the known number of optimal arms to efficiently allocate samples among tied arms.
	\item \textbf{Instance-Optimality Guarantee:} We prove that the modified algorithm achieves asymptotic instance-optimal sample complexity, matching the new lower bound and providing the first formal optimality guarantee for multi-optimal settings with known arm cardinality.
	\item \textbf{Insights for Efficient Exploration:} Our results illustrate how structural knowledge of optimal arm multiplicity can be exploited to design exploration strategies that are both theoretically grounded and practically efficient.
\end{itemize}
\subsection{Notations} 
For two probability distributions $P$ and $Q$ defined over a common support $\mathcal{X}$, the Kullback–Leibler (KL) divergence from $P$ to $Q$ is given by
\begin{equation}
D_{\mathrm{KL}}(P \,\|\, Q)
= \int_{\mathcal{X}} \log\!\left(\frac{dP}{dQ}(x)\right) dP(x),
\end{equation}
where $\frac{dP}{dQ}$ denotes the Radon–Nikodym derivative of $P$ with respect to $Q$. 
This formulation is valid for both continuous and discrete probability measures.

For any positive real number $x \in \mathbb{R}_+$, we use $\log(x)$ to denote the logarithm base $2$. 
Furthermore, the binary KL divergence between two Bernoulli parameters $x, y \in [0,1]$ is defined as
\begin{equation}
\mathrm{kl}(x, y)
= x \log \frac{x}{y} + (1 - x) \log \frac{1 - x}{1 - y},
\end{equation}
with the conventions $\mathrm{kl}(0,0) = \mathrm{kl}(1,1) = 0$.

For any differentiable function $f: \mathbb{R} \to \mathbb{R}$, we denote its first and second derivatives by
\begin{equation}
\dot{f}(x) = \frac{d f(x)}{d x}, 
\qquad 
\ddot{f}(x) = \frac{d^2 f(x)}{d x^2}.
\end{equation}

We denote the $K$-dimensional probability simplex by
\begin{equation}
\Sigma_K = \left\{ w \in \mathbb{R}_+^K : \sum_{i=1}^K w_i = 1 \right\}.
\end{equation}

Finally, for any $M \in \mathbb{N}^+$, we define the index set $[M] = \{1, 2, \ldots, M\}$.
\section{One-parameter exponential family}
The exponential family encompasses a broad and significant class of probability distributions, including many widely used models in statistics and machine learning, such as the Bernoulli, Poisson, Gaussian (with known variance), and exponential distributions. A particularly important subclass is the one-parameter exponential family, where each distribution is characterized by a single real-valued parameter.

A probability distribution \( P_{\mu} \) belongs to the one-parameter exponential family if its density (or mass) function can be written in the form:
\begin{align}
p_{\mu}(x) = h(x) \exp\left( \theta(\mu) x - A(\theta(\mu)) \right),
\end{align}
where $h(x)$ is a non-negative function, and
\begin{itemize}
\item \( \mu = \mathbb{E}[X] \) denotes the mean of the distribution,
\item \( \theta(\mu) \) is the natural (canonical) parameter, assumed to be a bijective function of \( \mu \),
\item \( A(\theta) \) is the log-partition function, which is convex and differentiable, and ensures normalization:
\begin{align*}
\int h(x) \exp\left( \theta x - A(\theta) \right) dx = 1,
\end{align*}
\item \( d(\mu, \lambda) := D_{\mathrm{KL}}(P_{\mu} \| P_{\lambda}) \) denotes the Kullback–Leibler (KL) divergence between two members of the family.
\end{itemize}

The KL divergence between \( P_{\mu} \) and \( P_{\lambda} \) admits the following closed-form expression:
\begin{align}
d(\mu, \lambda) = A(\theta(\lambda)) - A(\theta(\mu)) - (\theta(\lambda) - \theta(\mu)) \mu \label{defkl}.
\end{align}
Then, the following results can be proved.

\begin{lemma} \label{easy:lem1} For the one-parameter exponential family, it holds that
\begin{align}
\frac{d\theta(\mu)}{d\mu}=\frac{1}{\ddot{A}(\theta(\mu))}>0, \qquad \forall \mu. 
\end{align}
\end{lemma}
\begin{proof}
The mean of the distribution is given by:
\begin{align}
\mu(\theta)=\bbE_{\theta}[X]=\dot{A}(\theta). 
\end{align}
Since the function $A(\theta)$ is strictly convex, its derivative $\dot{A}(\theta)$, which maps $\theta \mapsto \mu$, is strictly increasing. 

From calculus, if a function $f$ is strictly increasing and differentiable, then its inverse $f^{-1}$ is differentiable, and
\begin{align}
\frac{d}{dy}f^{-1}(y)=\frac{1}{\dot{f}(f^{-1}(y)}.
\end{align}
Applying this to $f(\theta)=\mu(\theta)=\dot{A}(\theta)$, whose inverse is $\theta(\mu)$, we get:
\begin{align}
\frac{d}{d\mu}\theta(\mu)=\frac{1}{\ddot{A}(\theta(\mu))}.
\end{align}
Because $A(\theta)$ is the log-partition function of an exponential family, it is strictly convex, and hence:
\begin{align}
\ddot{A}(\theta)>0, \qquad \forall \theta.
\end{align}
Therefore: 
\begin{align}
\frac{d}{d\mu}\theta(\mu)=\frac{1}{\ddot{A}(\theta(\mu))}>0, \qquad \forall \mu. 
\end{align}
\end{proof}
The following result can be shown in Appendix \ref{easylem2proof}. 
\begin{lemma}\label{easy:lem2} Let $(w_1,w_2,\cdots,w_M, w_a)$ be a tuple of positive numbers. For common exponential family distributions such as Gaussian with known variance, Bernoulli, and Poisson, it holds that
\begin{align*}
f(\lambda_1,\lambda_2,\cdots, \lambda_M, \lambda_a)=\sum_{i=1}^M w_i d(\mu_i,\lambda_i)+w_a d(\mu_a,\lambda_a)
\end{align*} is jointly convex in $(\lambda_1,\lambda_2,\cdots, \lambda_M, \lambda_a)$, where $d(\mu,\lambda)$ is the KL divergence between two distributions from the exponential family defined in \eqref{defkl}.  
\end{lemma}
\section{Problem Setup}

We consider the stochastic multi-armed bandit (MAB) problem in the fixed-confidence best-arm identification setting. Let $\calA = \{1, 2, \dots, K\}$ be a finite set of $K$ arms. Each arm $a \in \mathcal{A}$ is associated with an unknown distribution $\nu_a$ belonging to a known exponential family, with mean reward $\mu_a = \mathbb{E}_{X \sim \nu_a}[X]$.

At each round $t \in \mathbb{N}$, the learner selects an arm $A_t \in \calA$ and observes an independent sample $X_t \sim \nu_{A_t}$. Based on the history of arm pulls and rewards, i.e., $\calF_t=\sigma(A_1,X_1,A_2,X_2,\cdots, A_t,X_t)$, the learner chooses a stopping time $\tau$ and outputs a recommendation $\hat{a}_\tau \in \mathcal{A}$.

Define the set of optimal arms as
\[
\mathcal{A}^\star := \left\{ a \in \mathcal{A} \,:\, \mu_a = \mu^\star \right\}, \quad \text{where} \quad \mu^\star := \max_{a \in \mathcal{A}} \mu_a.
\]
The goal is to identify any one arm in $\mathcal{A}^\star$ with high probability while minimizing the expected sample complexity.

\paragraph{Fixed-Confidence Guarantee.}
Given a confidence parameter $\delta \in (0,1)$, a $(\delta, \mathcal{A}^\star)$-PAC (Probably Approximately Correct) algorithm is one that ensures
\[
\mathbb{P}(\hat{a}_\tau \in \mathcal{A}^\star) \geq 1 - \delta.
\]

\paragraph{Sample Complexity}
The performance of an algorithm is measured by its sample complexity, i.e., the expected number of samples used before stopping:
$
\mathbb{E}[\tau].
$
Our aim is to design algorithms that satisfy the PAC criterion and achieve the minimal possible sample complexity, especially in the presence of multiple optimal arms.

\paragraph{Challenge with Multiple Optima}
When $|\mathcal{A}^\star| > 1$, standard algorithms may over-sample to distinguish between arms that are equally optimal. The key challenge is to avoid unnecessary comparisons among optimal arms and stop as soon as any one of them is confidently identified.

\section{General Lower Bound}
Let $N_a(t)=\sum_{s=1}^t \bone_{\{A_s=a\}}$ be the number of draws of arm $a$ between the instants $1$ and $t$, and $N_a=N_a(\tau)$ be the total number of draws of arm $a$.  
First, we recall the following lemma. 
\begin{lemma}\cite[Lemma 1]{Kaufmann2014OnTC}  \label{trans:lem} 
Let $\nu$ and $\nu'$ be two bandit models with $K$ arms such that for all $a$, the distributions $\nu_a$ and $\nu'_a$ are mutually absolutely continuous.  For any almost-surely finite stopping time $\sigma$ with respect to $(\calF_t)$,
\begin{align}
\sum_{a=1}^K \bbE_{\nu}[N_a(\sigma)] D_{\text{KL}}(\nu_a,\nu'_a) \geq \sup_{\calE \in \calF_{\sigma}} \text{kl}(\bbP_{\nu}(\calE), \bbP_{\nu'}(\calE)).
\end{align}
\end{lemma}
Now, denote by
\begin{align}
\calS=\big\{\mu=(\mu_1,\mu_2,\cdots,\mu_K): \mbox{exactly $M$ optimal arms}\big\}. 
\end{align}
\begin{definition} \label{defI}
Let $a \in \{1,2,\cdots,K\}\setminus \{1,2,\cdots,M\}$. For any tuple $(\alpha_1, \alpha_2, \cdots, \alpha_M) \in [0,1]^M$ such that $0\leq \sum_{i=1}^M \alpha_i\leq 1$, define
\begin{align}
&I_{\alpha_1,\alpha_2,\cdots,\alpha_M}(\nu_1,\nu_2,\cdots, \nu_M,\nu_a)\nn\\
&\quad =\sum_{i=1}^M \alpha_i d\bigg(
\nu_i, \sum_{i=1}^M \alpha_i \nu_i+\bigg(1-\sum_{i=1}^M \alpha_i\bigg)\nu_a\bigg)\nn\\
&\qquad +\bigg(1-\sum_{i=1}^M \alpha_i\bigg) d\bigg(\nu_a, \sum_{i=1}^M \alpha_i \nu_i+\bigg(1-\sum_{i=1}^M \alpha_i\bigg)\nu_a\bigg), \nn\\
&\qquad \forall (\nu_1,\nu_2,\cdots,\nu_M,\nu_a) \in \bbR^{M+1}. 
\end{align} 
\end{definition}
Then, the following can be proved. 
\begin{theorem} \label{thm:lowbound} Consider a bandit $\mu=(\mu_1,\mu_2,\cdots, \mu_K) \in \calS$ such that the set of optimal arms is $\mu_1,\mu_2,\cdots,\mu_M$. 
Let
\begin{align}
\text{Alt}(\mu)&=\bigcup_{a\in \calA \setminus \{1,2,\cdots,M\}}\big\{\lambda \in \calS: (\lambda_a > \lambda_1) \wedge (\lambda_a >\lambda_2) \nn\\
&\qquad \wedge \cdots \wedge (\lambda_a>\lambda_M)\}. 
\end{align}
Let $\delta \in (0,1)$. Then, for any $\delta$-PAC strategy, it holds that
\begin{align}
\bbE_{\mu} [\tau]\geq T^*(\mu) \text{kl}(1-\delta, \delta),
\end{align} where
\begin{align}
&T^*(\mu)^{-1}=\sup_{w \in \Sigma_K} \inf_{\lambda \in \text{Alt}(\mu)} \bigg(\sum_{a=1}^K w_a d(\mu_a,\lambda_a)\bigg)\label{eqtami1} \\
&=\sup_{w \in \Sigma_K}\min_{a \notin [M]} \bigg(\sum_{i=1}^M w_i+w_a\bigg)\nn\\
& \times I_{\frac{w_1}{\sum_{i=1}^M w_i+w_a}, \frac{w_2}{\sum_{i=1}^M w_i+w_a}, \cdots, \frac{w_M}{\sum_{i=1}^M w_i +w_a}}(\mu_1,\mu_2,\cdots, \mu_M, \mu_a) \label{eqtami2}. 
\end{align}
\end{theorem}
\begin{remark} Some remarks are in order.
\begin{itemize}
\item The case $M = K$ is trivial when $M$ is known in advance, as any arm can be selected. Hence, in this work we assume $K > M$. Moreover, Theorem~\ref{thm:lowbound} recovers the sample complexity result in \cite[Lemma 3]{garivier2016optimal} when $M = 1$.
\item As $\text{kl}(1-\delta, \delta) \sim \log(1/\delta)$ when $\delta$ goes to zero. Theorem \ref{thm:lowbound} provides an asymptotic lower bound
\begin{align}
\liminf_{\delta \to 0}\frac{\bbE_{\mu} [\tau]}{\log(1/\delta)} \geq T^*(\mu). 
\end{align}
A non-asymptotic version can be obtained by using the inequality $\text{kl}(1-\delta, \delta)\geq \log(1/(2.4\delta))$ that holds for any $\delta \in (0,1)$ \cite{garivier2016optimal}. 
\end{itemize}
\end{remark}
\begin{proof}
For any $\lambda \in \text{Alt}(\mu)$, there exists $a \notin [M]$ such that $\lambda_a> \lambda_i$ for all $i \in [M]$. Hence, $[M]$ is not the set of optimal arms under $\lambda$.  Thus, introducing the event $\calE=\{\hat{a}_{\tau} \in [M]\} \in \calF_{\tau}$, any $\delta$-PAC algorithm satisfies $\bbP_{\mu}(\calE)>1-
\delta$ and $\bbP_{\lambda}(\calE)\leq \delta$. 

By Lemma \ref{trans:lem}, we have
\begin{align}
\forall \lambda \in \text{Alt}(\mu):  \sum_{a=1}^K d(\mu_a,\lambda_a) \bbE_{\mu}[N_a(\tau)]\geq \text{kl}(1-\delta,\delta). 
\end{align} 
Then, it holds that
\begin{align}
 \text{kl}(1-\delta,\delta)\leq \bbE_{\mu}[\tau] \sup_{w \in \Sigma_K} \inf_{\lambda \in \text{Alt}(\mu)} \bigg(\sum_{a=1}^K w_a d(\mu_a,\lambda_a)\bigg), 
\end{align}
or
\begin{align}
 \bbE_{\mu}[\tau] \geq \bigg[\sup_{w \in \Sigma_K} \inf_{\lambda \in \text{Alt}(\mu)} \bigg(\sum_{a=1}^K w_a d(\mu_a,\lambda_a)\bigg)\bigg]^{-1} \text{kl}(1-\delta,\delta). 
\end{align}
Now, we estimate $ \inf_{\lambda \in \text{Alt}(\mu)} \big(\sum_{a=1}^K w_a d(\mu_a,\lambda_a)\big)$ for a fixed $w \in \Sigma_K$.  Observe that
\begin{align}
& \inf_{\lambda \in \text{Alt}(\mu)} \bigg(\sum_{a=1}^K w_a d(\mu_a,\lambda_a)\bigg)\nn\\
& \qquad =\min_{a \notin [M]} \inf_{\lambda: \lambda_a\geq \lambda_1, \lambda_a \geq \lambda_2, \cdots, \lambda_a \geq \lambda_M} \sum_{i=1}^M w_i d(\mu_i,\lambda_i)\nn\\
&\qquad \qquad + w_a d(\mu_a,\lambda_a).
\end{align}
Next, we minimize $f(\lambda_1,\lambda_2,\cdots, \lambda_M,\lambda_a)=\sum_{i=1}^M w_i d(\mu_i,\lambda_i)+ w_a d(\mu_a,\lambda_a)$ under constraints $\lambda_a \geq \lambda_1, \lambda_a \geq \lambda_2, \cdots, \lambda_a\geq  \lambda_M$. This is a convex optimization problem (see Lemma \ref{easy:lem2}) that can solved analytically. First, we assume that $w_i>0, \forall i \in [M]$. 

The Lagrangian is
\begin{align}
&L(\lambda_1,\lambda_2,\cdots,\lambda_M,\lambda_a, \beta_1,\beta_2, \cdots, \beta_M)\nn\\
&\quad =\sum_{i=1}^M w_i d(\mu_i,\lambda_i)+ w_a d(\mu_a,\lambda_a)- \sum_{i=1}^M \beta_i (\lambda_a-\lambda_i),
\end{align} where $\beta_i \geq 0$ for all $i \in [M]$.  

For a one-parameter exponential family with density
\[
p_\mu(x) = h(x) \exp\left( \theta(\mu)x - A(\theta(\mu)) \right),
\]
where $\mu = \dot{A}(\theta)$, the KL divergence between $P_\mu$ and $P_\lambda$ is given by:
\[
d(\mu, \lambda) = A(\theta(\lambda)) - A(\theta(\mu)) - (\theta(\lambda) - \theta(\mu)) \mu.
\]

Hence, we have
\begin{align}
&L(\lambda_1,\lambda_2,\cdots,\lambda_M,\lambda_a, \beta_1,\beta_2, \cdots, \beta_M)\nn\\
&\quad =\sum_{i=1}^M w_i \big[A(\theta(\lambda_i)) - A(\theta(\mu_i)) - (\theta(\lambda_i) - \theta(\mu_i)) \mu_i\big]\nn\\
&\quad \qquad +w_a \big[A(\theta(\lambda_a)) - A(\theta(\mu_a)) - (\theta(\lambda_a) - \theta(\mu_a)) \mu_a\big]\nn\\
&\qquad \qquad - \sum_{i=1}^M \beta_i (\lambda_a-\lambda_i). 
\end{align}

Then, the KKT conditions are
\begin{align}
&0=\frac{\partial L}{\partial \lambda_i}=w_i (\lambda_i-\mu_i)\frac{d\theta(\lambda_i)}{d\lambda_i}+\beta_i, \qquad \forall i \in [M], \label{eq14}\\
&0=\frac{\partial L}{\partial \lambda_a}=w_a (\lambda_a-\mu_a) \frac{d\theta(\lambda_a)}{d\lambda_a}-\sum_{i=1}^M \beta_i, \label{eq15}\\
&\beta_i(\lambda_a-\lambda_i)=0, \qquad \forall i \in [M], \label{eq16}\\
&\lambda_a \geq \lambda_i, \qquad \forall i \in [M] \label{eq17}. 
\end{align}
If $w_a = 0$, then from \eqref{eq15} we obtain $\beta_i = 0$ for all $i \in [M]$. Consequently, by \eqref{eq14}, it follows that $\lambda_i = \mu_i$ for all $i \in [M]$. In this case, we have
\begin{align}
f(\lambda_1, \lambda_2, \ldots, \lambda_M, \lambda_a)
= \sum_{i=1}^M w_i d(\mu_i, \lambda_i) + w_a d(\mu_a, \lambda_a)
= 0.
\end{align}
Therefore, without loss of generality, we may assume that $w_a > 0$.

We consider three cases:
\begin{itemize}
\item Case 1:  $\beta_i=0$ for all $i \in [M]$. Then, from \eqref{eq14}, \eqref{eq15}, and Lemma \ref{easy:lem1} we obtain $\lambda_i=\mu_i \quad \forall  i \in [M]$ and $\lambda_a=\mu_a$.  Then, we have $\lambda_a <\lambda_i$ for all $i \in [M]$, which contradicts \eqref{eq17}. 
\item Case 2: There exists exactly $L$ indices $\{i_1,i_2,\cdots, i_L\}$ for $1\leq L<M$ such that $\lambda_a=\lambda_{i_k}$ for all $k \in [L]$.  Then, we have
\begin{align}
w_{i_k} (\lambda_{i_k}-\mu_{i_k})\frac{d\theta(\lambda_{i_k})}{d\lambda_{i_k}}+\beta_{i_k}&=0,\qquad \forall k \in [L],\\
w_a (\lambda_a-\mu_a)\frac{d\theta(\lambda_a)}{d\lambda_a}-\sum_{k=1}^L \beta_{i_k}&=0
\end{align}
It follows that
\begin{align}
\sum_{k=1}^L w_{i_k} (\lambda_{i_k}-\mu_{i_k})\frac{d\theta(\lambda_{i_k})}{d\lambda_{i_k}}+w_a (\lambda_a-\mu_a)\frac{d\theta(\lambda_a)}{d\lambda_a}=0,
\end{align} which leads to (by combining with Lemma \ref{easy:lem1})
\begin{align}
\lambda_a=\lambda_{i_1}=\cdots =\lambda_{i_L}=\frac{\sum_{k=1}^L w_{i_k} \mu_{i_k}+ w_a \mu_a}{\sum_{k=1}^L w_{i_k}+w_a} \label{eq20}.
\end{align}
Now, since $L<M$, from \eqref{eq16} there exists $j \in [M]\setminus \{i_1,i_2,\cdots, i_L\}$ such that $\beta_j=0$. Hence, from \eqref{eq14} we have
\begin{align}
\lambda_j=\mu_j \label{eq21}.
\end{align}
Now, from \eqref{eq20} and \eqref{eq21} we have
$
\lambda_a< \lambda_j,
$ which contradicts \eqref{eq17}. 
\item Case 3: $\lambda_1=\lambda_2=\cdots =\lambda_M=\lambda_a$. Then, from \eqref{eq14} and \eqref{eq15} and using the same arguments as Case 2, we have
\begin{align}
\lambda_1=\lambda_2=\cdots =\lambda_M=\lambda_a=\frac{\sum_{i=1}^M w_i \mu_i+ w_a \mu_a}{\sum_{i=1}^M w_i+w_a}. 
\end{align}
Then, we have
\begin{align}
&\min_{a \notin [M]} \inf_{\lambda: \lambda_a\geq \lambda_1, \lambda_a \geq \lambda_2, \cdots, \lambda_a \geq \lambda_M} 
\nn\\
&\qquad \sum_{i=1}^M w_i d(\mu_i,\lambda_i)+ w_a d(\mu_a,\lambda_a)\nn\\
&=\min_{a \notin [M]} \bigg(\sum_{i=1}^M w_i+w_a\bigg)\nn\\
&\qquad \times I_{\frac{w_1}{\sum_{i=1}^M w_i+w_a},  \cdots, \frac{w_M}{\sum_{i=1}^M w_i+w_a}}(\mu_1,\cdots, \mu_M,\mu_a) \label{amet}.
\end{align}
\end{itemize}
In cases where a subset of $\{w_1, w_2, \ldots, w_M\}$ takes zero values, it is straightforward to verify that the optimal value of the problem remains unchanged from that in \eqref{amet}.

Finally, by letting
\begin{align}
T^*(\mu)^{-1}&= \sup_{w \in \Sigma_K}\min_{a \notin [M]} \inf_{\lambda: \lambda_a\geq \lambda_1, \lambda_a \geq \lambda_2, \cdots, \lambda_a \geq \lambda_M}\nn\\
&\qquad  \sum_{i=1}^M w_i d(\mu_i,\lambda_i)+ w_a d(\mu_a,\lambda_a),
\end{align}
and
\begin{align}
w^*(\mu)&=\text{argmax}_{w \in \Sigma_K}  \min_{a \notin [M]} \inf_{\lambda: \lambda_a\geq \lambda_1, \lambda_a \geq \lambda_2, \cdots, \lambda_a \geq \lambda_M} \nn\\
&\qquad  \sum_{i=1}^M w_i d(\mu_i,\lambda_i)+ w_a d(\mu_a,\lambda_a),
\end{align} we conclude our proof of this lemma. 
\end{proof}
\section{The Track-and-Stop Strategy} \label{sec:achieve}
\subsection{Sampling Rule: Tracking the Optimal Proportions}
We adopt the same sampling rules as described in \cite[Section 3]{garivier2016optimal}. The core idea is to match the optimal sampling proportions $w^*(\mu)$ by tracking the plug-in estimate $ w^*(\hat{\mu}(t))$. However, in bandit settings, relying directly on plug-in estimates can be risky: poor initial estimates may cause some arms to be prematurely discarded, thus preventing the collection of new observations that would correct earlier mistakes. In fact, naive plug-in tracking may fail to identify the best arm reliably.

To mitigate this issue, a simple yet effective solution is to enforce sufficient exploration across all arms to ensure fast convergence of the empirical means $\hat{\mu}(t)$. 

One such strategy, termed \emph{C-Tracking}, modifies the optimization step by projecting \( w^*(\mu) \) onto a truncated simplex to ensure every arm is sampled at least minimally. For each $ \varepsilon \in (0, 1/K]$, define the truncated simplex:
\begin{align}
\Sigma_K^{\varepsilon} = \left\{ (w_1, \dots, w_K) \in [\varepsilon, 1]^K : \sum_{a=1}^K w_a = 1 \right\}.
\end{align}
Let $w^\varepsilon(\mu)$ denote the $ L^\infty $-projection of $ w^*(\mu) $ onto $ \Sigma_K^{\varepsilon}$, and choose $\varepsilon_t = \frac{1}{2\sqrt{K^2 + t}}$. Then the sampling rule is given by
\begin{align}
A_{t+1} \in \arg\max_{1 \leq a \leq K} \sum_{s=0}^{t} w_a^{\varepsilon_s}(\hat{\mu}(s)) - N_a(t).
\end{align}

An alternative, often more practical rule is \emph{D-Tracking}, which targets \( w^*(\hat{\mu}(t)) \) directly and introduces forced exploration only when necessary. Let
\begin{align}
U_t = \left\{ a \in \{1, \dots, K\} : N_a(t) < \sqrt{t} - \frac{K}{2} \right\}.
\end{align}
Then the sampling rule is defined as
\begin{align}
A_{t+1} \in 
\begin{cases}
\arg\min_{a \in U_t} N_a(t) & \text{if } U_t \neq \emptyset, \\
\arg\max_{1 \leq a \leq K} \left[ t w^*_a(\hat{\mu}(t)) - N_a(t) \right] & \text{otherwise}.
\end{cases}
\end{align}

Both sampling rules are guaranteed to ensure that the empirical sampling proportions converge to the optimal allocation, as established in \cite[Section 3.1]{garivier2016optimal}.

\subsection{Chernoff's Stopping Rule}
Let
\begin{align}
&Z_{a; b_1,b_2,\cdots,b_M}(t)=\log\bigg( \nn\\
& \frac{\max_{\mu'_a\leq \mu'_{b_1}   \cdots \wedge \mu'_a \leq \mu'_{b_M} }p_{\mu'_a}(\underline{X}_{N_a(t)}^a)  \prod_{i=1}^M p_{\mu'_{b_i}}(\underline{X}_{N_{b_i}(t)}^{b_i}) }{\max_{\mu'_a\geq \mu'_{b_1}  \cdots \mu'_a\geq \mu'_{b_M} }p_{\mu'_a}(\underline{X}_{N_a(t)}^a)  \prod_{i=1}^M p_{\mu'_{b_i}}(\underline{X}_{N_{b_i}(t)}^{b_i})}\bigg) \label{estim},
\end{align}
where
\begin{align}
p_{\mu}(Z_1,Z_2,\cdots,Z_n)=\prod_{k=1}^n \exp\big(\dot{A}^{-1}(\mu) Z_k-A(\dot{A}^{-1}(\mu))\big). 
\end{align}

Observe that
\begin{align}
&p_{\mu'_a}(\underline{X}_{N_a(t)}^a)\nn\\
&\quad =\prod_{k=1}^{N_a(t)} \exp\bigg(\dot{A}^{-1}(\mu'_a)X_k-A(\dot{A}^{-1}(\mu'_a) )\bigg)\\
&\quad =\exp\bigg[N_a(t)\bigg(\dot{A}^{-1}(\mu'_a)\hat{\mu}_a(t)-A(\dot{A}^{-1}(\mu'_a))\bigg)\bigg].
\end{align}
Here,
\begin{align}
\hat{\mu}_a(t)=\frac{1}{N_a(t)} \sum_{k=1}^{N_a(t)} X_k. 
\end{align}
Similarly, we have
\begin{align}
&p_{\mu'_{b_i}}(\underline{X}_{N_{b_i}(t)}^{b_i})\nn\\
&\quad =\exp\bigg[N_{b_i}(t)\bigg(\dot{A}^{-1}(\mu'_{b_i})\hat{\mu}_{b_i}(t)-A(\dot{A}^{-1}(\mu'_{b_i}))\bigg)\bigg]
\end{align} for all $i \in [M]$.
Hence, we have
\begin{align}
&\log \max_{\mu'_a\leq \mu'_{b_1} \wedge \cdots \wedge \mu'_a \leq \mu'_{b_M} }p_{\mu'_a}(\underline{X}_{N_a(t)}^a)  \prod_{i=1}^M p_{\mu'_{b_i}}(\underline{X}_{N_{b_i}(t)}^{b_i}) \nn\\
&\qquad =\max_{\mu'_a\leq \mu'_{b_1} \wedge \cdots \wedge \mu'_a \leq \mu'_{b_M} }\bigg[l_a(\mu'_a)+\sum_{i=1}^M l_{b_i}(\mu'_{b_i})\bigg],
\end{align} where
\begin{align}
l_x(\mu'_x)=N_x(t)\bigg(\dot{A}^{-1}(\mu'_x)\hat{\mu}_x(t)-A(\dot{A}^{-1}(\mu'_x))\bigg),
\end{align} for all $x \in \{a,b_1,b_2,\cdots, b_M\}$, is concave in $\mu'_x$. 

Now, we have
\begin{align}
\frac{\partial l_x(\mu'_x)}{\partial \mu'_x}&=\frac{\partial l_x(t)}{\partial \dot{A}^{-1}(\mu'_x)}\frac{\partial \dot{A}^{-1}(\mu'_x)}{\partial \mu'_x}\\
&= N_x(t)(\hat{\mu}_x(t)-\mu'_x)\frac{\partial \dot{A}^{-1}(\mu'_x)}{\partial \mu'_x}\\
&=N_x(t) (\hat{\mu}_x(t)-\mu'_x)\bigg(\frac{1}{\ddot{A}(\dot{A}^{-1}(\mu'_x)}\bigg),
\end{align} which is equal to zero if and only if $\mu'_x=\hat{\mu}_x(t)$.  

This means that
\begin{align}
&\max_{\mu'_a\leq \mu'_{b_1}  \wedge \cdots \wedge \mu'_a \leq \mu'_{b_M} }\bigg[l_a(\mu'_a)+\sum_{i=1}^M l_{b_i}(\mu'_{b_i})\bigg]\nn\\
&\quad \leq l_a(\hat{\mu}_a(t))+ \sum_{i=1}^M l_{b_i}(\hat{\mu}_{b_i}(t))\label{eq28c}.
\end{align}
Under the condition that $\hat{\mu}_a(t)\leq \hat{\mu}_{b_i}(t)), \forall i \in [M]$, the equality in \eqref{eq28c} is achieved, which means that
\begin{align}
&\log \max_{\mu'_a\leq \mu'_{b_1}\wedge \cdots \wedge \mu'_a \leq \mu'_{b_M}}p_{\mu'_a}(\underline{X}_{N_a(t)}^a) \prod_{i=1}^M p_{\mu'_{b_i}}(\underline{X}_{N_{b_i}(t)}^{b_i})\nn\\
&\qquad= l_a(\hat{\mu}_a(t))+ \sum_{i=1}^M l_{b_i}(\hat{\mu}_{b_i}(t)) \label{eq29mod}.
\end{align}
Similarly, we have
\begin{align}
&\log \max_{\mu'_a\geq \mu'_{b_1} \wedge  \cdots \wedge \mu'_a \geq \mu'_{b_M}}p_{\mu'_a}(\underline{X}_{N_a(t)}^a) \prod_{i=1}^M p_{\mu'_{b_i}}(\underline{X}_{N_{b_i}(t)}^{b_i})\nn\\
&\qquad =\max_{\mu'_a\geq \mu'_{b_1}  \wedge \cdots \wedge \mu'_a \geq \mu'_{b_M}}\bigg[l_a(\mu'_a)+\sum_{i=1}^M l_{b_i}(\mu'_{b_i})\bigg],
\end{align} where
\begin{align}
l_x(\mu'_x)=N_x(t)\bigg(\dot{A}^{-1}(\mu'_x)\hat{\mu}_x(t)-A(\dot{A}^{-1}(\mu'_x))\bigg),
\end{align} for all $x \in \{a,b_1,b_2,\cdots,b_M\}$, is concave in $\mu'_x$.

Now, we aim to find 
\begin{align}
\max_{\mu'_a\geq \mu'_{b_1} \wedge \mu'_a \geq \mu'_{b_2} \wedge \cdots \wedge \mu'_a \geq \mu'_{b_M}}\bigg[l_a(\mu'_a)+\sum_{i=1}^M l_{b_i}(\mu'_{b_i})\bigg].
\end{align}
The Lagrangian is 
\begin{align}
&L(\mu'_a,\mu'_{b_1},\mu'_{b_2},\cdots,\mu'_{b_M},\beta_1,\beta_2,\cdots,\beta_M)\nn\\
&\quad =l_a(\mu'_a)+\sum_{i=1}^M l_{b_i}(\mu'_{b_i})+\sum_{i=1}^M \beta_i (\mu'_a-\mu'_{b_i})
\end{align} where $\beta_i \geq 0$ for all $i \in [M]$. 

Using the KKT conditions, we have
\begin{align}
0&=\frac{\partial L}{\partial \mu'_a}= N_a(t)\big(\hat{\mu}_a(t)-\mu'_a\big)\bigg[\frac{1}{\ddot{A}(\dot{A}^{-1}(\mu'_a))}\bigg]+\sum_{i=1}^M \beta_i \label{eq23},\\
0&=\frac{\partial L}{\partial \mu'_{b_i}}= N_{b_i}(t)\big(\hat{\mu}_{b_i}(t)-\mu'_{b_i}\big)\bigg[\frac{1}{\ddot{A}(\dot{A}^{-1}(\mu'_{b_i}))}\bigg]-\beta_i  \label{eq24},\\
0&=\beta_i(\mu'_a-\mu'_{b_i}), \qquad \forall i \in [M] \label{eq25}\\
\mu'_a&\geq \mu'_{b_i}, \qquad \forall i \in [M] \label{eq26}.
\end{align}
Now, let $\calU=\{i \in [M]: \beta_i\neq 0\}$ and $\calV=\{i \in [M]: \beta_i =0\}$. We consider two cases:
\begin{itemize}
\item Case 1: $\calV =\emptyset$. Then, we have $\calU=[M]$. 
From \eqref{eq25}, we must have
\begin{align}
\mu'_a=\mu'_{b_1}=\mu'_{b_2}=\cdots= \mu'_{b_M} \label{eq29}.
\end{align}
Now, by summing all equations in \eqref{eq23} and \eqref{eq24} we obtain
\begin{align}
 &N_a(t)\big(\hat{\mu}_a(t)-\mu'_a\big)\bigg[\frac{1}{\ddot{A}(A^{-1}(\mu'_a))}\bigg]\nn\\
 &\quad + \sum_{i=1}^M N_{b_i}(t)\big(\hat{\mu}_{b_i}(t)-\mu'_{b_i}\big)\bigg[\frac{1}{\ddot{A}(A^{-1}(\mu'_{b_i}))}\bigg]=0 \label{eq30}. 
\end{align}
From \eqref{eq29} and \eqref{eq30}, we obtain
\begin{align}
&\mu'_a=\mu'_{b_1}=\mu'_{b_2}=\cdots= \mu'_{b_M}=\hat{\mu}_{a;b_1,b_2,\cdots, b_M}(t)\nn\\
&\quad :=\frac{N_a(t) \hat{\mu}_a(t)+\sum_{i=1}^M N_{b_i}(t)\hat{\mu}_{b_i}(t)}{N_a(t)+\sum_{i=1}^M N_{b_i}(t)}.
\end{align}
Hence, we have
\begin{align}
&\max_{\mu'_a\geq \mu'_{b_1} \wedge \mu'_a \geq \mu'_{b_2} \wedge \cdots \wedge \mu'_a \geq \mu'_{b_M} }\bigg[l_a(\mu'_a)+\sum_{i=1}^M l_{b_i}(\mu'_{b_i})\bigg]\nn\\
&\quad =l_a(\hat{\mu}_{a;b_1,b_2,\cdots, b_M}(t))+\sum_{i=1}^M l_{b_i} (\hat{\mu}_{a;b_1,b_2,\cdots, b_M}(t)) \label{eq31m}
\end{align}
By combining \eqref{eq29mod} and \eqref{eq31m}, under the condition $(\hat{\mu}_a(t)\leq \hat{\mu}_{b_1}(t))\wedge (\hat{\mu}_a(t)\leq \hat{\mu}_{b_2}(t))\wedge \cdots \wedge (\hat{\mu}_a(t)\leq \hat{\mu}_{b_M}(t))$, we obtain
\begin{align}
&Z_{a;b_1,b_2,\cdots,b_M}(t)=N_a(t)d(\hat{\mu}_a(t),\hat{\mu}_{a;b_1,b_2,\cdots,b_M}(t))\nn\\
&\quad +\sum_{i=1}^M N_{b_i}(t) d(\hat{\mu}_{b_i}(t),\hat{\mu}_{a;b_1,b_2,\cdots,b_M}(t)) \label{eqktm}.
\end{align}
\item Case 2: $\calV \neq \emptyset$. Then, for all $i \in \calU$, we have $\mu'_{b_i}= \mu'_a$ by \eqref{eq25}. Then, from \eqref{eq24} we have
\begin{align}
\sum_{i \in \calU} N_{b_i}(t)\big(\hat{\mu}_{b_i}(t)-\mu'_{b_i}\big)\bigg[\frac{1}{\ddot{A}(\dot{A}^{-1}(\mu'_{b_i}))}\bigg]-\beta_i=0 \label{eq45}.
\end{align}
Hence, from \eqref{eq23} and \eqref{eq45} we obtain
\begin{align}
&N_a(t)\big(\hat{\mu}_a(t)-\mu'_a\big)\bigg[\frac{1}{\ddot{A}(\dot{A}^{-1}(\mu'_a))}\bigg]\nn\\
&\quad + \sum_{i \in \calU} N_{b_i}(t)\big(\hat{\mu}_{b_i}(t)-\mu'_a \big)\bigg[\frac{1}{\ddot{A}(\dot{A}^{-1}(\mu'_a))}\bigg]=0 \label{eq46}.
\end{align}
From \eqref{eq46} we obtain
\begin{align}
\mu'_a=\frac{N_a(t) \hat{\mu}_a(t)+\sum_{i \in \calU}  N_{b_i}(t) \hat{\mu}_{b_i}(t)}{N_a(t)+\sum_{i \in \calU}  N_{b_i}(t)  }:=\bar{\mu}_a(t) \label{eq47}. 
\end{align}
Now, from \eqref{eq24} we obtain
\begin{align}
\mu'_{b_i}= \hat{\mu}_{b_i}(t), \qquad \forall i \in \calB \label{eq26a}.
\end{align}
Then, by \eqref{eq26} and \eqref{eq26a} we have
\begin{align}
\bar{\mu}_a(t) \geq \max_{i \in \calV} \mu'_{b_i}= \max_{i \in \calV}  \hat{\mu}_{b_i}(t) \geq \frac{\sum_{i\in \calV} N_{b_i}(t) \hat{\mu}_{b_i}(t)}{\sum_{i \in \calV} N_{b_i}(t)} \label{eq47c}. 
\end{align}
Now, using the fact that for any $p,q,r,s>0$ if $\frac{p}{q}\leq \frac{r}{s}$ we have $\frac{p}{q} \leq \frac{p+r}{q+s} \leq \frac{r}{s}$, from \eqref{eq47} and \eqref{eq47c} we obtain
\begin{align}
 \frac{\sum_{i\in \calV} N_{b_i}(t) \hat{\mu}_{b_i}(t)}{\sum_{i \in \calV} N_{b_i}(t)}\leq \hat{\mu}_{a;b_1,b_2,\cdots, b_M}(t) \leq \bar{\mu}_a(t) \label{aq10}. 
\end{align}
Observe that
\begin{align}
\frac{\sum_{i\in \calV} N_{b_i}(t) \hat{\mu}_{b_i}(t)}{\sum_{i \in \calV} N_{b_i}(t)}\geq \min_{i \in \calV} \hat{\mu}_{b_i}(t)\geq \min_{i \in [M]} \hat{\mu}_{b_i}(t) \label{aq11}. 
\end{align}
On the other hand, from \eqref{eq24} and $\beta_i \geq 0 \enspace \forall i \in [M]$, we have
\begin{align}
\bar{\mu}_a(t)=\mu'_a \leq \min_{i \in \calU} \hat{\mu}_{b_i}(t)  \label{aq12}. 
\end{align}
It follows from \eqref{aq10}, \eqref{aq11}, and \eqref{aq12} that
\begin{align}
 &\min_{i \in [M]} \hat{\mu}_{b_i}(t) \leq \hat{\mu}_{a;b_1,b_2,\cdots, b_M}(t) \leq \bar{\mu}_a(t)\nn\\
 &\quad  \leq \min_{i \in \calU} \hat{\mu}_{b_i}(t) \leq \max_{i \in [M]} \hat{\mu}_{b_i}(t) \label{aq13}.
\end{align}
Therefore, we obtain
\begin{align}
 \hat{\mu}_{a;b_1,b_2,\cdots, b_M}(t) \leq \bar{\mu}_a(t) \leq  \hat{\mu}_{a;b_1,b_2,\cdots, b_M}(t)+ \zeta(t) \label{aq14},
\end{align} where
\begin{align}
\zeta(t):= \max_{i \in [M]} \hat{\mu}_{b_i}(t) - \min_{i \in [M]} \hat{\mu}_{b_i}(t).
\end{align}
\end{itemize}
Hence, we have
\begin{align}
&\max_{\mu'_a\geq \mu'_{b_1} \wedge \mu'_a \geq \mu'_{b_2} \wedge \cdots \wedge \mu'_a \geq \mu'_{b_M} }\bigg[l_a(\mu'_a)+\sum_{i=1}^M l_{b_i}(\mu'_{b_i})\bigg] \nn\\
&\quad = l_a(\bar{\mu}_a(t))+\sum_{i\in \calU} l_{b_i} (\bar{\mu}_a(t)) +\sum_{i \in \calV} l_{b_i}(\hat{\mu}_{b_i}(t))\label{eq31h}.
\end{align}
Under the condition $\hat{\mu}_a (t) \leq \min_{i \in [M]} \hat{\mu}_{b_i}(t)$, it holds that
\begin{align}
 \hat{\mu}_{a;b_1,b_2,\cdots, b_M}(t)\geq \hat{\mu}_a(t). 
\end{align}
Hence, by using the fact that $l_x(\mu'_x)$ is non-increasing if $\mu'_x \in [\hat{\mu}_x(t), \infty)$ and non-decreasing if $\mu'_x \in (-\infty, \hat{\mu}_x(t)])$, from \eqref{eq31h} and \eqref{aq14} we obtain
\begin{align}
&\max_{\mu'_a\geq \mu'_{b_1} \wedge  \cdots \wedge \mu'_a \geq \mu'_{b_M} }\bigg[l_a(\mu'_a)+\sum_{i=1}^M l_{b_i}(\mu'_{b_i})\bigg]\nn\\
&\quad  \geq l_a(\hat{\mu}_{a;b_1,b_2,\cdots, b_M}(t)+ \zeta(t) )\nn\\
&\qquad +\sum_{i \in \calU}l_{b_i} (\hat{\mu}_{a;b_1,b_2,\cdots, b_M}(t))+\sum_{i \in \calV} l_{b_i}(\hat{\mu}_{b_i}(t)) \label{eq31k}.
\end{align}
By combining \eqref{eq29mod} and \eqref{eq31k}, under the condition $(\hat{\mu}_a(t)\leq \hat{\mu}_{b_1}(t))\wedge (\hat{\mu}_a(t)\leq \hat{\mu}_{b_2}(t))\wedge \cdots \wedge (\hat{\mu}_a(t)\leq \hat{\mu}_{b_M}(t))$, we obtain
\begin{align}
&Z_{a;b_1,b_2,\cdots,b_M}(t) \nn\\
&\quad \leq N_a(t)d(\hat{\mu}_a(t),\hat{\mu}_{a;b_1,b_2,\cdots,b_M}(t)+\zeta(t))\nn\\
&\qquad +\sum_{i\in \calU} N_{b_i}(t) d(\hat{\mu}_{b_i}(t),\hat{\mu}_{a;b_1,b_2,\cdots,b_M}(t))\\
&\quad = N_a(t)d(\hat{\mu}_a(t),\hat{\mu}_{a;b_1,b_2,\cdots,b_M}(t))\nn\\
&\qquad +\sum_{i=1}^M N_{b_i}(t) d(\hat{\mu}_{b_i}(t),\hat{\mu}_{a;b_1,b_2,\cdots,b_M}(t))+ R(t),
\end{align}
where
\begin{align}
&R(t)=N_a(t)\big[d(\hat{\mu}_a(t),\hat{\mu}_{a;b_1,b_2,\cdots,b_M}(t)+\zeta(t))\nn\\
&\quad - d(\hat{\mu}_a(t),\hat{\mu}_{a;b_1,b_2,\cdots,b_M}(t))\big]\nn\\
&\qquad -\sum_{i \in \calV}  N_{b_i}(t) d(\hat{\mu}_{b_i}(t),\hat{\mu}_{a;b_1,b_2,\cdots,b_M}(t))\\
&\quad \leq t\bigg(\frac{N_a(t)}{t} \big[d(\hat{\mu}_a(t),\hat{\mu}_{a;b_1,b_2,\cdots,b_M}(t)+\zeta(t))\nn\\
&\qquad - d(\hat{\mu}_a(t),\hat{\mu}_{a;b_1,b_2,\cdots,b_M}(t))\big]\nn\\
&\qquad  -\min_{i \in [M]} \frac{N_{b_i}(t)}{t} d(\hat{\mu}_{b_i}(t),\hat{\mu}_{a;b_1,b_2,\cdots,b_M}(t))\bigg)\\
&\quad  \leq 0 \label{hayt}
\end{align} under the condition that
\begin{align}
&\frac{N_a(t)}{t} \big[d(\hat{\mu}_a(t),\hat{\mu}_{a;b_1,b_2,\cdots,b_M}(t)+\zeta(t))\nn\\
&\quad - d(\hat{\mu}_a(t),\hat{\mu}_{a;b_1,b_2,\cdots,b_M}(t))\big]\nn\\
&\qquad  -\min_{i \in [M]} \frac{N_{b_i}(t)}{t} d(\hat{\mu}_{b_i}(t),\hat{\mu}_{a;b_1,b_2,\cdots,b_M}(t))\leq 0. 
\end{align}
From \eqref{eqktm} and \eqref{hayt}, we finally obtain
\begin{align}
&Z_{a;b_1,b_2,\cdots,b_M}(t)\nn\\
&\quad =N_a(t)d(\hat{\mu}_a(t),\hat{\mu}_{a;b_1,b_2,\cdots,b_M}(t))\nn\\
&\qquad +\sum_{i=1}^M N_{b_i}(t) d(\hat{\mu}_{b_i}(t),\hat{\mu}_{a;b_1,b_2,\cdots,b_M}(t)) \label{eqktm2},
\end{align} under the condition that $(\hat{\mu}_a(t)\leq \hat{\mu}_{b_1}(t))\wedge (\hat{\mu}_a(t)\leq \hat{\mu}_{b_2}(t))\wedge \cdots \wedge (\hat{\mu}_a(t)\leq \hat{\mu}_{b_M}(t))$ and
\begin{align}
&S_{a;b_1,b_2,\cdots,b_M}(t)\nn\\
&\quad := \frac{N_a(t)}{t} \big[d(\hat{\mu}_a(t),\hat{\mu}_{a;b_1,b_2,\cdots,b_M}(t)+\zeta(t))\nn\\
&\qquad - d(\hat{\mu}_a(t),\hat{\mu}_{a;b_1,b_2,\cdots,b_M}(t))\big]\nn\\
&\qquad  -\min_{i \in [M]} \frac{N_{b_i}(t)}{t} d(\hat{\mu}_{b_i}(t),\hat{\mu}_{a;b_1,b_2,\cdots,b_M}(t))\leq 0. 
\end{align}
Next, we use the following stopping rule
\begin{align}
\tau
&=\inf\bigg\{t \in \bbN: Z(t):= \max_{b_1,b_2,\cdots,b_M}\min_{a \in  \calA \setminus \{b_1,b_2,\cdots,b_M\}}  \nn\\
&\qquad Z_{a;b_1,b_2,\cdots,b_M}(t)>\beta(t,\delta)\bigg\},
\end{align} where $\beta(t,\delta)$ is an exploration rate to be adjusted. 

The decoding rule is as following:
\begin{itemize}
\item Let 
\begin{align}
&(\hat{b}_1,\hat{b}_2,\cdots,\hat{b}_M)=\text{argmax}_{b_1,b_2,\cdots,b_M} \min_{a \in \calA \setminus \{b_1,b_2,\cdots,b_M\}} \nn\\
&\quad Z_{a;b_1,b_2,\cdots,b_M}(\tau).
\end{align}.
\item Choose $\hat{a}_{\tau}=\hat{b}_i$ with probability $1/M$ for all $i \in [M]$. 
\end{itemize} 
\section{Choosing the threshold in the Stopping Rule}
For simplicity, we only provide a proof for Bernoulli bandits (see \cite{garivier2016optimal}). The extension to other exponential bandits can be done by using the same methods as \cite{Chambaz2005AMA}.

Introducing, for \(a,b_1,b_2,\cdots,b_M \in \mathcal{A}\),
\begin{align*}
T_{a;b_1,b_2,\cdots,b_M} = \inf\{t \in \mathbb{N} : Z_{a;b_1,b_2,\cdots,b_M}(t) > \beta(t,\delta)\}.
\end{align*}
Then, we have
\begin{align}
&\mathbb{P}_{\mu}\big(\tau<\infty, \hat{a}_{\tau} \notin [M]\big) \nonumber \\
&\quad \leq \mathbb{P}_{\mu}\Big(\exists \{b_1,b_2,\cdots,b_M\} \neq [M], \nonumber \\
&\quad \quad \exists t \in \mathbb{N}: \min_{a \in \mathcal{A} \setminus \{b_1,b_2,\cdots,b_M\}} Z_{a;b_1,b_2,\cdots,b_M}(t) > \beta(t,\delta) \Big) \label{eq54}.
\end{align}
Now, for each tuple \((b_1,b_2,\cdots,b_M)\) such that \(\{b_1,b_2,\cdots,b_M\} \neq [M]\), there exists
$
a^*_{b_1,b_2,\cdots,b_M} \in [M] \quad \text{such that} \quad a^*_{b_1,b_2,\cdots,b_M} \in \mathcal{A} \setminus \{b_1,b_2,\cdots,b_M\}.
$
Then, we have
\[
\mu_{a^*_{b_1,b_2,\cdots,b_M}} \geq \max_{1 \leq i \leq M} \mu_{b_i}.
\]
In addition, from \eqref{eq54} we obtain
\begin{align}
&\mathbb{P}_{\mu}\big(\tau<\infty, \hat{a}_{\tau} \notin [M]\big) \nonumber \\
&\qquad \leq \sum_{b_1,b_2,\cdots,b_M} \mathbb{P}_{\mu}\big(T_{a^*_{b_1,b_2,\cdots,b_M};b_1,b_2,\cdots,b_M} < \infty\big).
\end{align}

Now, it is sufficient to show that if
\[
\beta(t,\delta) = \log \left[ K^M \left(\frac{4}{M+1}\right)^{(M+1)/2} t^{(M+1)/2} \frac{1}{\delta} \right]
\]
and
\[
\mu_a \geq \mu_{b_1} \wedge \mu_a \geq \mu_{b_2} \wedge \cdots \wedge \mu_a \geq \mu_{b_M},
\]
then
\[
\mathbb{P}_{\mu}\big(T_{a;b_1,b_2,\cdots,b_M} < \infty\big) \leq \frac{\delta}{K^M}.
\]

For such a tuple of arms, observe that on the event \(\{T_{a;b_1,b_2,\cdots,b_M} = t\}\), time \(t\) is the first moment when \(Z_{a;b_1,b_2,\cdots,b_M}\) exceeds the threshold \(\beta(t,\delta)\), which implies by definition that
\begin{align}
&1 \leq e^{-\beta(t,\delta)} \nn\\
&\times  \frac{
\max_{\mu'_a \leq \mu'_{b_1} \wedge \cdots \wedge \mu'_a \leq \mu'_{b_M}} p_{\mu'_a}(\underline{X}_{N_a(t)}^a) \prod_{i=1}^M p_{\mu'_{b_i}}(\underline{X}_{N_{b_i}(t)}^{b_i})
}{
\max_{\mu'_a \geq \mu'_{b_1} \wedge \cdots \wedge \mu'_a \geq \mu'_{b_M}} p_{\mu'_a}(\underline{X}_{N_a(t)}^a) \prod_{i=1}^M p_{\mu'_{b_i}}(\underline{X}_{N_{b_i}(t)}^{b_i})
}\Bigg]. 
\end{align}

It thus holds that
\begin{align}
&\mathbb{P}_{\mu}\big(T_{a;b_1,b_2,\cdots,b_M} < \infty\big) \nonumber \\
&= \sum_{t=1}^{\infty} \mathbb{P}_{\mu}\big(T_{a;b_1,b_2,\cdots,b_M} = t\big) = \sum_{t=1}^{\infty} \mathbb{E}_{\mu}\left[\mathbbm{1}_{(T_{a;b_1,b_2,\cdots,b_M} = t)}\right] \\
&\leq \sum_{t=1}^{\infty} e^{-\beta(t,\delta)} \mathbb{E}_{\mu}\Bigg[ \mathbbm{1}_{(T_{a;b_1,b_2,\cdots,b_M} = t)}\nn\\
&\quad  \frac{
\max_{\mu'_a \leq \mu'_{b_1} \wedge \cdots \wedge \mu'_a \leq \mu'_{b_M}} p_{\mu'_a}(\underline{X}_{N_a(t)}^a) \prod_{i=1}^M p_{\mu'_{b_i}}(\underline{X}_{N_{b_i}(t)}^{b_i})
}{
\max_{\mu'_a \geq \mu'_{b_1} \wedge \cdots \wedge \mu'_a \geq \mu'_{b_M}} p_{\mu'_a}(\underline{X}_{N_a(t)}^a) \prod_{i=1}^M p_{\mu'_{b_i}}(\underline{X}_{N_{b_i}(t)}^{b_i})
}\Bigg] \\
&\leq \sum_{t=1}^{\infty} e^{-\beta(t,\delta)} \mathbb{E}_{\mu}\Bigg[ \mathbbm{1}_{(T_{a;b_1,b_2,\cdots,b_M} = t)}\nn\\
&\quad  \frac{
\max_{\mu'_a \leq \mu'_{b_1} \wedge \cdots \wedge \mu'_a \leq \mu'_{b_M}} p_{\mu'_a}(\underline{X}_{N_a(t)}^a) \prod_{i=1}^M p_{\mu'_{b_i}}(\underline{X}_{N_{b_i}(t)}^{b_i})
}{
p_{\mu_a}(\underline{X}_{N_a(t)}^a) \prod_{i=1}^M p_{\mu_{b_i}}(\underline{X}_{N_{b_i}(t)}^{b_i})
}\Bigg].
\end{align}

Expanding the expectation into an integral over the sample space \(\{0,1\}^t\), we get
\begin{align}
&= \sum_{t=1}^{\infty} e^{-\beta(t,\delta)} \int_{\{0,1\}^t} \mathbbm{1}_{(T_{a;b_1,b_2,\cdots,b_M} = t)}(x_1,x_2,\cdots,x_t) \nonumber \\
&\quad \times
\max_{\mu'_a \leq \mu'_{b_1} \wedge \cdots \wedge \mu'_a \leq \mu'_{b_M}} p_{\mu'_a}(\underline{x}_{N_a(t)}^a) \prod_{i=1}^M p_{\mu'_{b_i}}(\underline{x}_{N_{b_i}(t)}^{b_i})\nn\\
&\qquad  \prod_{i \in \mathcal{A} \setminus \{a,b_1,\cdots,b_M\}} p_{\mu_i}(\underline{x}_{N_i(t)}^i)
dx_1 dx_2 \cdots dx_t.
\end{align}

Now, by using \cite[Lemma 11]{garivier2016optimal}, we have
\begin{align}
&\max_{\mu'_a \leq \mu'_{b_1} \wedge \cdots \wedge \mu'_a \leq \mu'_{b_M}} p_{\mu'_a}(\underline{x}_{N_a(t)}^a) \prod_{i=1}^M p_{\mu'_{b_i}}(\underline{x}_{N_{b_i}(t)}^{b_i}) \nonumber \\
&\quad \leq 2 \sqrt{N_a(t)} \prod_{i=1}^M (2 \sqrt{N_{b_i}(t)}) \, \text{kt}(\underline{x}_{N_a(t)}^a) \prod_{i=1}^M \text{kt}(\underline{x}_{N_{b_i}(t)}^{b_i}) \nonumber \\
&\quad \leq 2^{M+1} \left(\frac{t}{M+1}\right)^{\frac{M+1}{2}} \text{kt}(\underline{x}_{N_a(t)}^a) \prod_{i=1}^M \text{kt}(\underline{x}_{N_{b_i}(t)}^{b_i}) \label{mutat},
\end{align}
where \(\text{kt}(\cdot)\) is the Krichevsky–Trofimov distribution (see \cite{Krichevsky1981}). Here, we use the fact that $\nu_1 \nu_2 \cdots \nu_{M+1} \leq (\frac{\nu_1+\nu_2+\cdots+ \nu_{M+1}}{M+1})^{M+1}$ for all $\nu_1,\nu_2, \cdots, \nu_{M+1} \geq 0$ in \eqref{mutat}. 

Therefore, we obtain
\begin{align}
&\mathbb{P}_{\mu}\big(T_{a;b_1,b_2,\cdots,b_M} < \infty\big) \nn \\
&\quad \leq \sum_{t=1}^\infty 2^{M+1} \left(\frac{t}{M+1}\right)^{\frac{M+1}{2}} e^{-\beta(t,\delta)} \nn\\
&\qquad \times \int_{\{0,1\}^t} \mathbbm{1}_{(T_{a;b_1,b_2,\cdots,b_M} = t)}(x_1,\cdots,x_t) \nonumber \\
&\qquad \times \text{kt}(\underline{x}_{N_a(t)}^a) \prod_{i=1}^M \text{kt}(\underline{x}_{N_{b_i}(t)}^{b_i}) \nn\\
&\qquad \times \prod_{i \in \mathcal{A} \setminus \{a,b_1,\cdots,b_M\}} p_{\mu_i}(\underline{x}_{N_i(t)}^i) \, dx_1 \cdots dx_t.
\end{align}

Defining \(\tilde{\mathbb{P}}\) as the probability measure induced by the product of KT distributions for arms \(a,b_1,\cdots,b_M\) and the true distributions for all other arms, we have
\begin{align}
&\mathbb{P}_{\mu}\big(T_{a;b_1,b_2,\cdots,b_M} < \infty\big)\nn\\
&\leq \sum_{t=1}^\infty 2^{M+1} \left(\frac{t}{M+1}\right)^{\frac{M+1}{2}} e^{-\beta(t,\delta)} \tilde{\mathbb{P}}\big(T_{a;b_1,b_2,\cdots,b_M} = t\big) \nonumber \\
&\leq \frac{\delta}{K^M} \tilde{\mathbb{P}}\big(T_{a;b_1,b_2,\cdots,b_M} < \infty\big)\nn\\
& \leq \frac{\delta}{K^M}.
\end{align}
\section{Sample Complexity Analysis}
In this section, we analyze the sample complexity of the modified C-Tracking and D-Tracking algorithms introduced in Section~\ref{sec:achieve}. We begin by establishing the following result.
\begin{lemma} \label{easy:lem} 
Let \( \alpha > 0 \) and \( c_1, c_2 > 0 \) be such that
\[
A := \frac{1}{\alpha} \log c_2 + \log \left( \frac{\alpha}{c_1} \right) > \log e.
\]
Then there exists
\[
x \leq \frac{\alpha}{c_1} \left( A + \sqrt{2(A - 1)} \right)
\]
such that
\[
c_1 x \geq \log(c_2 x^{\alpha}).
\]
\end{lemma}
Based on Lemma~\ref{easy:lem}, the following subsections derive upper bounds on the sample complexity, which asymptotically match the lower bound established in Theorem~\ref{thm:lowbound}. 
\subsection{Almost-sure Upper Bound}
\begin{theorem}
Let $\mu$ denote a bandit model from an exponential family. Define
\begin{align*}
C = K^M \left( \frac{4}{M+1} \right)^{(M+1)/2}, \quad \alpha = \frac{M+1}{2},
\end{align*}
Consider applying Chernoff's stopping rule with confidence parameter 
\[
\beta(t, \delta) = \log \left( \frac{C t^{\alpha}}{\delta} \right),
\]
in conjunction with either the C-Tracking or D-Tracking sampling rule. Then, almost surely,
\[
\limsup_{\delta \to 0} \frac{\tau}{\log(1/\delta)} \leq T^*(\mu),
\]
where $\tau$ denotes the stopping time, and $T^*(\mu)$ is defined in Theorem \ref{thm:lowbound}.
\end{theorem}
\begin{proof}
For sufficiently large $ t$, i.e., $t \geq t_0$, it holds that
\begin{align}
(\hat{\mu}_1(t) \geq \hat{\mu}_a(t)) \vee (\hat{\mu}_2(t) \geq \hat{\mu}_a(t)) \wedge \cdots \wedge (\hat{\mu}_M(t) \geq \hat{\mu}_a(t))
\end{align}
for all $ a \notin [M]$.

In addition, by using the C-Tracking and D-Tracking sampling rule, for $t_1$ sufficiently large ($t_1>t_0$), for any $t\geq t_1$ and any $a \notin [M]$, we also have
\begin{align}
&\max_{a \notin [M]} S_{a;1,2,\cdots,M}(t) \nn \\
&\quad =\max_{a \notin [M]} \frac{N_a(t)}{t} \big[d(\hat{\mu}_a(t),\hat{\mu}_{a;1,2,\cdots,M}(t)+\zeta(t))\nn\\
&\qquad - d(\hat{\mu}_a(t),\hat{\mu}_{a;1,2,\cdots,M}(t))\big]\nn\\
&\qquad -\min_{i \in [M]} \frac{N_i(t)}{t} d(\hat{\mu}_i(t),\hat{\mu}_{a;1,2,\cdots,M}(t)) \nn \\
&\quad  \leq -\frac{1}{2}\min_{a \notin [M]} \min_{i \in [M]}  w_i^* d\bigg(\mu_i, \frac{\sum_{i=1}^M w_i^* \mu_i +w_a^* \mu_a}{\sum_{i=1}^M w_i^*+ w_a^*}\bigg)<0. 
\end{align}

Now for $t>t_1$, observe that
\begin{align}
Z(t) &\geq \min_{a \notin [M]} Z_{a;1,\ldots,M}(t) \nn \\
&= \min_{a \notin [M]} \sum_{i=1}^M N_i(t)\, d\left(\hat{\mu}_i(t), \hat{\mu}_{a;1,\ldots,M}(t)\right) \nn\\
&\quad + N_a(t)\, d\left(\hat{\mu}_a(t), \hat{\mu}_{a;1,\ldots,M}(t)\right) \nn \\
&= t  \min_{a \notin [M]} \left( \sum_{i=1}^M \frac{N_i(t)}{t} + \frac{N_a(t)}{t} \right) \nn \\
&\quad \times  I_{\frac{N_1(t)/t}{\sum_{i=1}^M N_i(t)/t + N_a(t)/t}, \ldots, \frac{N_M(t)/t}{\sum_{i=1}^M N_i(t)/t + N_a(t)/t}} \nn\\
&\qquad \left( \hat{\mu}_1(t), \ldots, \hat{\mu}_M(t), \hat{\mu}_a(t) \right).
\end{align}

For $a \geq M+1$, the function
\begin{align*}
&(w, \lambda) \mapsto \left(\sum_{i=1}^M w_i + w_a \right) \nn\\
&\qquad  \times I_{\frac{w_1}{\sum_{i=1}^M w_i + w_a}, \ldots, \frac{w_M}{\sum_{i=1}^M w_i + w_a}}(\lambda_1, \ldots, \lambda_M, \lambda_a)
\end{align*}
is continuous at $(w^*(\mu), \mu)$. Therefore, for any $ \varepsilon > 0$, there exists $t_2 \geq t_1$ such that for all $t \geq t_2$ and all $a \in \{M+1, \ldots, K\}$,
\begin{align}
&\left( \sum_{i=1}^M \frac{N_i(t)}{t} + \frac{N_a(t)}{t} \right) \nn \\
&\qquad \times  I_{\frac{N_1(t)/t}{\sum_{i=1}^M N_i(t)/t + N_a(t)/t}, \ldots, \frac{N_M(t)/t}{\sum_{i=1}^M N_i(t)/t + N_a(t)/t}} \nn\\
&\qquad \left( \hat{\mu}_1(t), \ldots, \hat{\mu}_M(t), \hat{\mu}_a(t) \right) \nn \\
&\quad \geq \frac{\sum_{i=1}^M w_i^* + w_a^*}{1+\varepsilon} \nn\\
&\qquad \times  I_{\frac{w_1^*}{\sum_{i=1}^M w_i^* + w_a^*}, \ldots, \frac{w_M^*}{\sum_{i=1}^M w_i^* + w_a^*}} (\mu_1, \ldots, \mu_M, \mu_a).
\end{align}

Hence, for $t \geq t_2$,
\begin{align}
Z(t) &\geq \frac{t}{1+\varepsilon} \cdot \min_{a \notin [M]} \left(\sum_{i=1}^M w_i^* + w_a^*\right) \nn\\
&\qquad \times  I_{\frac{w_1^*}{\sum_{i=1}^M w_i^* + w_a^*}, \ldots, \frac{w_M^*}{\sum_{i=1}^M w_i^* + w_a^*}} (\mu_1, \ldots, \mu_M, \mu_a) \nn\\
&= \frac{t}{(1+\varepsilon) T^*(\mu)}.
\end{align}

Therefore,
\begin{align}
\tau &= \inf \{ t \in \mathbb{N} : Z(t) \geq \beta(t, \delta) \} \nn \\
&\leq t_2 \vee \inf \left\{ t \in \mathbb{N} : \frac{t}{1+\varepsilon} \cdot \frac{1}{T^*(\mu)} \geq \log\left(\frac{C t^{\alpha}}{\delta}\right) \right\},
\end{align}
where
\[
C = K^M \left( \frac{4}{M+1} \right)^{(M+1)/2}, \quad \alpha = \frac{M+1}{2}.
\]
In the case where \( M \) is fixed (not a function of \( \delta \)), applying the Lemma \ref{easy:lem} with \( c_1 = \frac{1}{(1+\varepsilon) T^*(\mu)} \) and \( c_2 = \frac{C}{\delta} \) yields
\[
A = \frac{1}{\alpha} \log \frac{C}{\delta} + \log\left( \alpha T^*(\mu)(1+\varepsilon) \right) > \log e
\]
as \( \delta \to 0 \). Therefore,
\[
\limsup_{\delta \to 0} \frac{\tau}{\log(1/\delta)} \leq (1 + \varepsilon) T^*(\mu).
\]
Letting \( \varepsilon \to 0 \) concludes the proof.

\end{proof}
\subsection{Asymptotic Optimality in Expectation}
\begin{theorem}
Let $\mu$ be a bandit model from an exponential family.  Define
\begin{align*}
C = K^M \left( \frac{4}{M+1} \right)^{(M+1)/2}, \quad \alpha = \frac{M+1}{2},
\end{align*}
Using Chernoff's stopping rule with
\begin{align*}
\beta(t, \delta) = \log\left( \frac{C t^{\alpha}}{\delta} \right),
\end{align*} 
and the sampling rule C-Tracking or D-Tracking, we have
\begin{align}
\limsup_{\delta \to 0} \frac{\mathbb{E}_{\mu}[\tau_{\delta}]}{\log(1/\delta)} \leq T^*(\mu),
\end{align} where $T^*(\mu)$ is defined in Theorem \ref{thm:lowbound}. 
\end{theorem}
\begin{remark} Some remarks are in order.
\begin{itemize}
\item Consider the special Gaussian bandit $\calN(\mu,\sigma^2)$ where $\mu_1=\mu_2=\cdots =\mu_M=\Delta$ and $\mu_{M+1}=\mu_{M+2}=\cdots=\mu_K=0$ for some $\Delta>0$. Then, for any $a \notin [M]$ we have
\begin{align}
& I_{\frac{w_1}{\sum_{i=1}^M w_i+w_a},  \cdots, \frac{w_M}{\sum_{i=1}^M w_i +w_a}}(\mu_1,\cdots, \mu_M, \mu_a)\nn\\
&\quad =\bigg[\frac{\beta_a(w) (1-\beta_a(w))^2}{2\sigma^2} \Delta^2+ \frac{(1-\beta_a(w)) \beta_a^2(w)}{2\sigma^2} \Delta^2\bigg] \nn \\
&\quad=\frac{\beta_a(w) (1-\beta_a(w))}{2\sigma^2} \Delta^2,
\end{align} where
\begin{align}
\beta_a=\frac{\sum_{i=1}^M w_i}{\sum_{i=1}^M w_i +w_a}. 
\end{align}
Hence, by \eqref{eqtami2} it holds that
\begin{align}
T^*(\mu)&=\frac{2\sigma^2 }{\sup_{w \in \Sigma_K} \min_{a \notin [M]} \beta_a(w)(1-\beta_a(w)) \Delta^2}\nn\\
&=\Theta\bigg(\frac{1}{\Delta^2}\bigg).
\end{align}
This implies that the expected sample complexity diverges to infinity as $\Delta \to 0$ for $\Delta > 0$. When $\Delta = 0$ and $M = K$ is known, any arm can be selected as mentioned above, and hence the sample complexity is finite (in fact, zero).
\item For each bandit instance $\lambda \in \mathcal{S}$, let $i^*(\lambda)$ denote the set of all optimal arms. Further, for each $i \in \{1,2,\dots,M\}$, define
\begin{align*}
\neg i := \{\lambda \in \mathcal{S} : i \notin i^*(\lambda)\} = \bigcup_{a \neq i} \{\lambda \in \mathcal{S} : \lambda_a > \lambda_i\}.
\end{align*}
It is straightforward to verify that $\neg i \supset \text{Alt}(\mu)$ for all $i$. Therefore, by \cite[Theorem 1]{Degenne2019PureEW} and \eqref{eqtami1}, we have
\begin{align*}
T^*(\mu) \leq D(\mu)^{-1}, 
\end{align*}
where
\begin{align}
 D(\mu) = \max_{i \in i^*(\mu)} \max_{w \in \Sigma_K} \inf_{\lambda \in \neg i} \sum_{a=1}^K w_a \, d(\mu_a, \lambda_a).
 \end{align}
Here, $D(\mu)^{-1}$ corresponds to the sample complexity of the same bandit problem in the unknown-$M$ setting \cite{Degenne2019PureEW}. This highlights that knowing $M$ can improve the sample complexity. Technically, the stopping time based on the generalized log-likelihood ratio (GLLR) statistic $Z_{a; b_1, b_2, \dots, b_M}(t)$ in \eqref{estim} is only well-defined when $M$ is known (see Section \ref{sec:achieve}).
\end{itemize}
\end{remark}
\begin{proof}
Following a similar strategy as in~\cite[Theorem 14]{garivier2016optimal}, we simplify notation by assuming that the bandit means satisfy
\[
\mu_1 = \mu_2 = \cdots = \mu_M > \mu_{M+1} \geq \mu_{M+2} \geq \cdots \geq \mu_K.
\]
Fix $\varepsilon > 0$. By the continuity of $w^*(\mu)$, there exists $\xi = \xi(\varepsilon) \leq \frac{\mu_M - \mu_{M+1}}{4}$ such that the hypercube
\[
\mathcal{I}_{\varepsilon} := [\mu_1 - \xi, \mu_1 + \xi] \times \cdots \times [\mu_K - \xi, \mu_K + \xi]
\]
satisfies, for all $\mu' \in \mathcal{I}_{\varepsilon}$,
\[
\max_a \left| w_a^*(\mu') - w_a^*(\mu) \right| \leq \varepsilon.
\]
In particular, if $\hat{\mu}(t) \in \mathcal{I}_\varepsilon$, then the top-$M$ empirical arms are exactly $\{1,2,\dots,M\}$.

Let $T \in \mathbb{N}$, define $h(T) := T^{1/4}$, and introduce the event
\[
\mathcal{E}_T := \bigcap_{t=h(T)}^T \{ \hat{\mu}(t) \in \mathcal{I}_\varepsilon \}.
\]
In addition, by Lemma~\cite[Lemma 20]{garivier2016optimal}, for any $\eps>0$ there exists a constant $T_{\eps}$ such that for $T\geq T_{\eps}$, it holds on $\calE_T$, for either C-Tracking or D-Tracking,
\begin{align}
\max_a \bigg|\frac{N_a(t)}{t}-w_a^*(\mu)\bigg| \leq 3 (K-1)\eps, \qquad \forall t \geq \sqrt{T}. 
\end{align}
Hence, by using the C-Tracking and D-Tracking sampling rule, for any sufficient small $\eps$, there exists $T^*_{\eps}> T_{\eps}$  such that for any $T>T_{\eps}^*$ we have
\begin{align}
&\sup_{a\notin [M]} S_{a;1,2,\cdots,M}(t)\nn\\
&\quad =\max_{a \notin [M]} \frac{N_a(t)}{t} \big[d(\hat{\mu}_a(t),\hat{\mu}_{a;1,2,\cdots,M}(t)+\zeta(t))\nn\\
&\qquad - d(\hat{\mu}_a(t),\hat{\mu}_{a;1,2,\cdots,M}(t))\big]\nn\\
&\quad \qquad  -\min_{i \in [M]} \frac{N_i(t)}{t} d(\hat{\mu}_i(t),\hat{\mu}_{a;1,2,\cdots,M}(t)) \nn \\
&\quad \leq 0, \qquad \forall t \geq \sqrt{T}. 
\end{align}

On the event $\mathcal{E}_T$, for $t \geq h(T)$ and any $a \notin [M]$, we have $\hat{\mu}_i(t) \geq \hat{\mu}_a(t)$ for all $i \in [M]$ and $S(t) \leq 0$, and the Chernoff-type statistic satisfies
\begin{align}
&\max_{b_1,\dots,b_M} \min_{a \notin \{b_1,\dots,b_M\}} Z_{a; b_1,\dots,b_M}(t)\nn\\
&\quad \geq \min_{a \notin [M]} Z_{a;1,\dots,M}(t) \\
&\quad = \min_{a \notin [M]} \big[ N_a(t) d(\hat{\mu}_a(t), \hat{\mu}_{a;1,\dots,M}(t)) \nn\\
&\qquad + \sum_{i=1}^M N_i(t) d(\hat{\mu}_i(t), \hat{\mu}_{a;1,\dots,M}(t)) \big] \\
&\quad = t \cdot g\left(\hat{\mu}(t), \left\{ \frac{N_i(t)}{t} \right\}_{i=1}^K \right),
\end{align}
where
\begin{align}
g(\mu', w') &:= \min_{a \notin [M]} \left( \sum_{i=1}^M w'_i + w'_a \right)\nn\\
&\qquad \times I_{\left\{ \frac{w'_i}{\sum_{j=1}^M w'_j + w'_a} \right\}_{i=1}^M}(\mu'_1,\dots,\mu'_M, \mu'_a).
\end{align}

Using~\cite[Lemma 20]{garivier2016optimal}, for $T \geq T_\varepsilon^*$, define
\[
C^*_\varepsilon(\mu) := \inf_{\substack{\|\mu' - \mu\|_\infty \leq \xi(\varepsilon)\\ \|w' - w^*(\mu)\|_\infty \leq 3(K-1)\varepsilon}} g(\mu', w'),
\]
so that on $\mathcal{E}_T$, for all $t \geq \sqrt{T}$,
\[
\max_{b_1,\dots,b_M} \min_{a \notin \{b_1,\dots,b_M\}} Z_{a; b_1,\dots,b_M}(t) \geq t C^*_\varepsilon(\mu).
\]
Then, for $T \geq T_\varepsilon^*$, on $\mathcal{E}_T$,
\begin{align}
&\min(\tau, T)\nn\\
&\leq \sqrt{T} + \sum_{t = \sqrt{T}}^T \mathbbm{1}_{\tau > t} \\
&\leq \sqrt{T} \nn\\
&\quad + \sum_{t = \sqrt{T}}^T \mathbbm{1}\left\{ \max_{b_1,\dots,b_M} \min_{a \notin \{b_1,\dots,b_M\}} Z_{a; b_1,\dots,b_M}(t) \leq \beta(t,\delta) \right\} \\
&\leq \sqrt{T} + \sum_{t = \sqrt{T}}^T \mathbbm{1}\{ t C^*_\varepsilon(\mu) \leq \beta(T,\delta) \} \\
&\leq \sqrt{T} + \frac{\beta(T,\delta)}{C^*_\varepsilon(\mu)}.
\end{align}

Define
\[
T_0(\delta) := \inf\left\{ T \in \mathbb{N} : \sqrt{T} + \frac{\beta(T,\delta)}{C^*_\varepsilon(\mu)} \leq T \right\}.
\]
Then, for any $T \geq \max(T_0(\delta), T_\varepsilon^*)$,
\[
\mathcal{E}_T \subset \{ \tau_\delta \leq T \},
\]
and by~\cite[Lemma 19]{garivier2016optimal},
\[
\mathbb{P}_\mu(\tau_\delta > T) \leq \mathbb{P}(\mathcal{E}_T^c) \leq B T \exp(-C T^{1/8}).
\]
Hence,
\[
\mathbb{E}_\mu[\tau] \leq T_0(\delta) + T_\varepsilon^* + \sum_{T=1}^\infty B T \exp(-C T^{1/8}).
\]

We now derive an upper bound on $T_0(\delta)$. For any $\eta > 0$, define
\[
C(\eta) := \inf\left\{ T \in \mathbb{N} : T - \sqrt{T} \geq \frac{T}{1 + \eta} \right\},
\]
then
\begin{align}
&T_0(\delta)\nn\\
&\quad \leq C(\eta) + \inf\left\{ T \in \mathbb{N} : \frac{1}{C^*_\varepsilon(\mu)} \log\left( \frac{C T^\alpha}{\delta} \right) \leq \frac{T}{1 + \eta} \right\} \\
&\quad \leq C(\eta) + \inf\left\{ T \in \mathbb{N} : \frac{C^*_\varepsilon(\mu)}{1 + \eta} T \geq \log\left( \frac{C T^\alpha}{\delta} \right) \right\}.
\end{align}
Applying Lemma~\ref{easy:lem} with $c_1 = \frac{C^*_\varepsilon(\mu)}{1 + \eta}$ and $c_2 = \frac{C}{\delta}$ yields
\[
\lim_{\delta \to 0} \frac{\mathbb{E}_\mu[\tau]}{\log \frac{1}{\delta}} \leq \frac{1 + \eta}{C^*_\varepsilon(\mu)}.
\]
Letting $\eta, \varepsilon \to 0$ and using continuity of $g$ and the definition of $w^*$ gives
\[
\lim_{\varepsilon \to 0} C^*_\varepsilon(\mu) = \frac{1}{T^*(\mu)},
\]
and thus,
\[
\lim_{\delta \to 0} \frac{\mathbb{E}_\mu[\tau]}{\log \frac{1}{\delta}} \leq T^*(\mu).
\]
\end{proof}
\section{Conclusion}

In this work, we revisited the problem of best-arm identification in stochastic multi-armed bandits under the fixed-confidence setting, focusing on the general case where multiple optimal arms may exist. Unlike prior studies, which primarily addressed the scenario with an unknown number of optimal arms, we investigated the setting where this number is known in advance.

Our main contribution lies in deriving a new information-theoretic lower bound that characterizes the minimal achievable sample complexity when the number of optimal arms is known. Building on this insight, we established that a modified Track-and-Stop algorithm—equipped with an appropriate stopping rule based on generalized log-likelihood ratio statistics—achieves this bound asymptotically. This result provides the first instance-optimal guarantee for Track-and-Stop in multi-optimal settings and demonstrates that knowledge of the number of optimal arms can fundamentally reduce sample complexity.

Overall, our findings extend the theoretical foundations of best-arm identification and reinforce the role of likelihood ratio-based methods as principled and optimal strategies for sequential decision-making under uncertainty. Future research directions include generalizing these results to structured bandit models, such as combinatorial or contextual settings, and developing adaptive algorithms that leverage structural information while preserving optimality guarantees.

\appendices 
\section{Proof of Lemma \ref{easy:lem2}} \label{easylem2proof}
Let $P_{\mu}$ and $P_{\lambda}$ be distributions in a one-parameter exponential family. Then the KL divergence has the form:
\begin{align}
&d(\mu,\lambda)\nn\\
&\quad =D(P_{\mu}\|P_{\lambda})\\
&\quad =A(\theta(\lambda))-A(\theta(\mu))-(\theta(\lambda)-\theta(\mu))\mu.
\end{align} 
First, we show that for a fixed $\mu$, the function $d(\mu,\lambda)$ is convex in $\lambda$ for the common distributions above. 
\begin{itemize}
\item Gaussian Distribution (Known Variance):

Let $X \sim \mathcal{N}(\mu, \sigma^2)$. The natural parameter is $\theta(\lambda) = \lambda/\sigma^2$, and the log-partition function is $A(\theta) = \frac{\sigma^2 \theta^2}{2}$.

The KL divergence is:
\[
d(\mu, \lambda) = \frac{(\mu - \lambda)^2}{2\sigma^2}.
\]

This is a quadratic function in $\lambda$, and hence convex. In fact,
\[
\frac{d^2}{d\lambda^2} d(\mu, \lambda) = \frac{1}{\sigma^2} > 0.
\]

\item Bernoulli Distribution:

Let $X \sim \text{Bernoulli}(\lambda)$ with $\lambda \in (0,1)$. The KL divergence is:
\[
d(\mu, \lambda) = \mu \log\left(\frac{\mu}{\lambda}\right) + (1-\mu) \log\left(\frac{1-\mu}{1-\lambda}\right).
\]

Taking the second derivative with respect to $\lambda$, we obtain:
\[
\frac{d^2}{d\lambda^2} d(\mu, \lambda) = \frac{\mu}{\lambda^2} + \frac{1 - \mu}{(1 - \lambda)^2}.
\]

Since $\lambda \in (0,1)$ and $\mu \in (0,1)$, this expression is strictly positive. Thus, $d(\mu, \lambda)$ is convex in $\lambda$.

\item Poisson Distribution:

Let $X \sim \text{Poisson}(\lambda)$ with $\lambda > 0$. The KL divergence is:
\[
d(\mu, \lambda) = \lambda - \mu + \mu \log\left(\frac{\mu}{\lambda}\right).
\]

Taking the first and second derivatives:
\[
\frac{d}{d\lambda} d(\mu, \lambda) = 1 - \frac{\mu}{\lambda}, \quad \frac{d^2}{d\lambda^2} d(\mu, \lambda) = \frac{\mu}{\lambda^2} > 0.
\]

Hence, $d(\mu, \lambda)$ is convex in $\lambda$ for Poisson distributions.

\end{itemize}

In all three cases—Gaussian, Bernoulli, and Poisson—the KL divergence $d(\mu, \lambda)$ is convex in the second argument $\lambda$, which is a useful property in many optimization and information-theoretic contexts.

Therefore, $\{d(\mu_i,\lambda_i)\}_{i=1}^M$ and $d(\mu_a,\lambda_a)$ are convex in their respective arguments. 

Now, define
\begin{align}
f(\lambda_1,\lambda_2,\cdots, \lambda_M, \lambda_a)=\sum_{i=1}^M w_i d(\mu_i,\lambda_i)+ w_a d(\mu_a,\lambda_a).
\end{align}
Since each term $d(\mu_i,\lambda_i)$  is convex in $\lambda_i$, and since $w_1, w_2,\cdots,w_M, w_a>0$, their weighted sum is jointly convex in $(\lambda_1,\lambda_2,\cdots, \lambda_M, \lambda_a)$. This is a standard result from convex analysis. 
\section{Proof of Lemma \ref{easy:lem}}
First, we find $x$ such that
\begin{align}
c_1 x= \log (c_2 x^{\alpha})=\log c_2 + \alpha \log x. 
\end{align}
This equation is equivalent to
\begin{align}
c_1 x- \alpha \log x = \log c_2 \label{eq100}.
\end{align}
Let $x=\frac{\alpha}{c_1} y$, then \eqref{eq100} is equivalent to
\begin{align}
y=\log y+ A,
\end{align}
or
\begin{align}
(-y) e^{-y}=-e^{-A}.
\end{align}
Now, since $A\geq \log e$, we have $- e^{-A}\in [-\frac{1}{e},0)$. This means that
\begin{align}
y=-W_{-1}(-e^{-A})
\end{align} where $W_{-1}(\cdot)$ is the Lambert W function. 

Now, by \cite{Chatzigeorgiou2013BoundsOT}, it holds that
\begin{align}
-1-\sqrt{2u}-u<W_{-1}(-e^{-u-1})<-1-\sqrt{2u}-\frac{2}{3}u 
\end{align} for all $u>0$. It follows that
\begin{align}
y=-W_{-1}(-e^{-A}) \leq A+\sqrt{2(A-1)}, \qquad \text{if}\quad  A>1. 
\end{align} which leads to
\begin{align}
x\leq \frac{\alpha}{c_1}\big(A+\sqrt{2(A-1)}\big). 
\end{align}
\subsection*{Acknowledgements} The author would like to thank Ho Chi Minh City University of Technology (HCMUT), Vietnam National University, Ho Chi Minh City (VNU-HCM), for supporting this study. 
The author also sincerely thanks Jonathan Scarlett for his insightful suggestions, which significantly improved the manuscript. Appreciation is further extended to the Associate Editor and the reviewers for their constructive comments and valuable feedback, which helped enhance the overall quality of the paper.
\bibliographystyle{IEEEtran}
\bibliography{isitbib}
\begin{IEEEbiographynophoto}{Lan V. Truong} (S'12-M'15-SM'24) received his B.S.E. in Electronics and Telecommunications from Posts and Telecommunications Institute of Technology (PTIT), Vietnam, in 2003, his M.S.E. from the School of Electrical and Computer Engineering at Purdue University in 2011, and his Ph.D. in Electrical and Computer Engineering from National University of Singapore (NUS) in 2018.

He previously worked as an Operation and Maintenance Engineer at MobiFone Telecommunications Corporation in Hanoi. In 2012, he was a Research Assistant with the NSF Center for Science of Information and the Department of Computer Science at Purdue University. From 2013 to 2015, he served as an Academic Lecturer in the Department of Information Technology Specialization at FPT University.

Following his Ph.D., he was a Research Fellow in the Department of Computer Science, School of Computing at NUS (2018–2019). Since 2020, he has been a Research Associate in the Department of Engineering at the University of Cambridge. From 2023 to 2025, he held a Lecturer (Assistant Professor) position at the School of Mathematics, Statistics and Actuarial Science at the University of Essex.

Since September 2025, he has been a Lecturer (Assistant Professor) at the Faculty of Computer Science and Engineering, Ho Chi Minh City University of Technology (HCMUT), Vietnam National University Ho Chi Minh City (VNU-HCM). His research interests include information theory, machine learning, and high-dimensional statistics.
\end{IEEEbiographynophoto}
\end{document}